\pdfoutput=1

\documentclass[11pt]{article}

\usepackage{emnlp2021}

\usepackage{times}
\usepackage{latexsym}

\usepackage[T1]{fontenc}

\usepackage[utf8]{inputenc}

\usepackage{microtype}

\usepackage{booktabs}
\usepackage{amsfonts}
\usepackage{nicefrac}
\usepackage{amsmath}
\usepackage{amssymb}
\usepackage{graphicx}
\usepackage{algorithm}
\usepackage{tikz}
\usepackage[capitalise]{cleveref}
\usepackage{multirow}
\usepackage{color}
\usepackage{todonotes}
\newcommand{\boldx}{\mathbf{x}}

\newcommand{\boldy}{\mathbf{y}}

\newcommand{\boldW}{\mathbf{W}}

\newcommand{\btheta}{\boldsymbol{\theta}}

\newcommand{\niceeq}{\,{=}\,}
\newcommand{\nicein}{\,{\in}\,}

\newcommand{\mcX}{\mathcal{X}}
\newcommand{\mcV}{\mathcal{V}}

\newcommand{\mcN}{\mathcal{N}}

\newcommand{\mcA}{\mathcal{A}}

\newcommand{\reals}{\mathbb{R}}
\newcommand{\naturals}{\mathbb{N}}

\newcommand{\given}{\,|\,}

\DeclareMathOperator*{\ginsert}{insert}
\newcommand{\stopsymb}{\langle\text{stop}\rangle}

\newcommand{\neset}{\{\nu^{(n)}_{1:T_n}\}_{n=1}^N}
\newcommand{\actset}{\mcA_{\mathrm{ins}}}

\newcommand{\gimodel}{\textsc{Full}}
\newcommand{\ltwort}{\textsc{LRT}}
\usepackage{amsthm}
\newtheorem{claim}{Claim}

\newtheorem{proposition}{Proposition}
\newtheorem{observation}{Observation}
\newcommand{\eps}{\ensuremath{\varepsilon}}

\title{Data-to-text Generation by Splicing Together Nearest Neighbors}

\author{Sam Wiseman\thanks{\, This work was done primarily while the first two authors were at TTI-Chicago.} \\ Duke University \\ {\small \texttt{swiseman@cs.duke.edu}}
 \And Arturs Backurs \\ Microsoft Research \\ {\small \texttt{arturs.backurs@microsoft.com}}
 \And Karl Stratos \\ Rutgers University \\ {\small \texttt{karl.stratos@rutgers.edu}}}

\begin{document}
\maketitle
\begin{abstract}
We propose to tackle data-to-text generation tasks by directly splicing together retrieved segments of text from ``neighbor'' source-target pairs. Unlike recent work that conditions on retrieved neighbors but generates text token-by-token, left-to-right, we learn a policy that directly manipulates segments of neighbor text, by inserting or replacing them in partially constructed generations. Standard techniques for training such a policy require an oracle derivation for each generation, and we prove that finding the shortest such derivation can be reduced to parsing under a particular weighted context-free grammar. We find that policies learned in this way perform on par with strong baselines in terms of automatic and human evaluation, but allow for more interpretable and controllable generation.
\end{abstract}

\section{Introduction}
There has been recent interest in text generation systems that make use of retrieved ``neighbors'' --- examples of good text retrieved from a database, perhaps paired with the source information on which these example texts condition --- 
on the hope that these neighbors might make a generation task easier, or the system more interpretable or controllable~\citep[\textit{inter alia}]{song2016two,weston2018retrieve,guu2018generating,zhang2018guiding,peng2019text}.

Whereas most work along these lines has adopted a conventional encoder-decoder approach, conditioning on the retrieved neighbors and then autoregressively generating text token-by-token from left to right, we instead propose to generate text by directly splicing together segments of text from retrieved neighbors. 
Generating in this way aligns with the intuition that in settings such as data-to-text generation it ought to be sufficient to retrieve sentences similar to the one that must be generated, and then merely change some details, such as names or dates. 

There are notable advantages to a generation-by-splicing approach. First, generation becomes more interpretable: it is always clear from which neighbor a particular piece of generated text derives, and it is also clear how these pieces have come together to form the generated text. Generation-by-splicing may also increase our control over the generated text, and we suspect that approaches that make clear the provenance of each piece of generated text (as ours does) will be useful in preventing text generation systems from emitting harmful or biased text~\citep[\textit{inter alia}]{sheng2019woman,wallace2019universal,gehman2020realtoxicityprompts}. That is, we might imagine preventing systems from emitting harmful or biased text by only allowing generation from approved neighbor examples. 

Methodologically, we implement this generation-by-splicing approach by training a policy to directly insert or replace spans of neighbor text at arbitrary positions within a partially constructed generation, and we define a generalized $\ginsert$ function capable of such manipulations in Section~\ref{sec:splicing}. We train
this policy with ``teacher forcing''~\citep{williams1989learning}, which requires, for each training example, an oracle sequence of $\ginsert$ actions that derive it. Accordingly, we define a shortest sequence of actions deriving a training generation from its neighbors to be an oracle one, and we prove that, given some neighbors, an oracle sequence of actions can be obtained by parsing under a particular weighted context-free grammar, introduced in Section~\ref{sec:WCFG}.


Empirically, we find our proposed approach yields text of comparable quality to strong baselines under automatic metrics and human evaluation on the E2E dataset~\citep{novikova2017e2e} and Wikibio datasets~\citep{lebret2016neural}, but with added interpretability and controllability (see Section~\ref{sec:interp}). Our reduction of minimum-insertion generation to WCFG parsing may also be of independent interest. Our code is available at \url{https://github.com/swiseman/neighbor-splicing}.



\begin{figure*}[t!]
    \centering
    \includegraphics[scale=0.63]{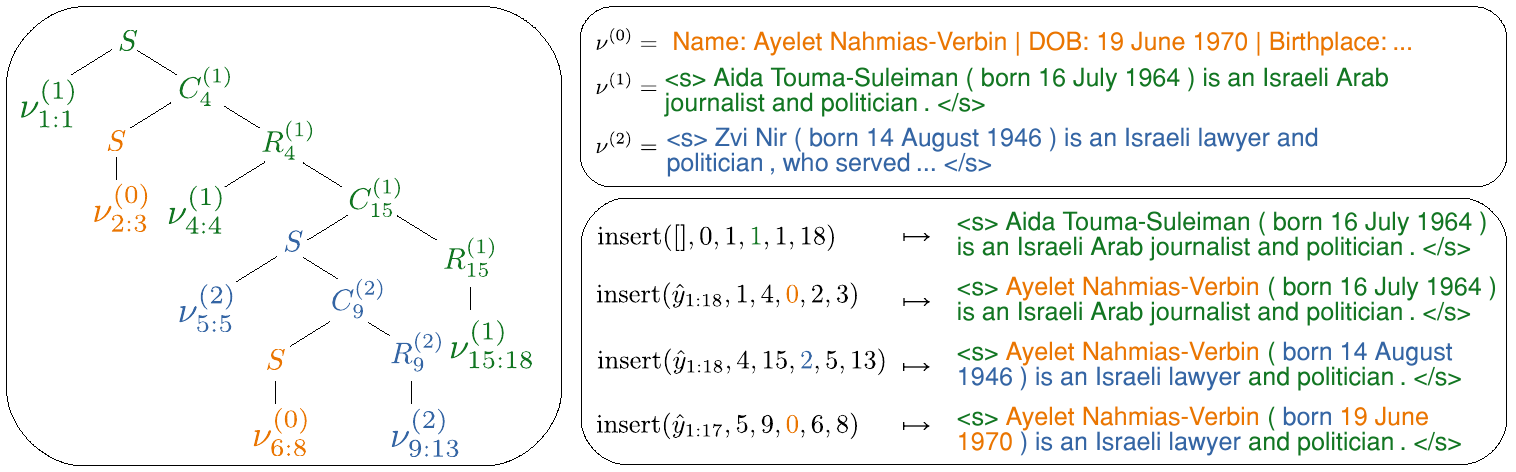}
    \caption{Deriving a sentence from the WikiBio dataset, ``Ayelet Nahmias-Verbin (born 19 June 1970) is an Israeli lawyer and politician.'' Top right: neighbor sequences $\nu^{(0)}, \nu^{(1)}, \nu^{(2)}$; $\nu^{(0)}$ is from the corresponding table. Bottom right: a sequence of $\ginsert$ operations (see Equation~\eqref{eq:insert}) deriving the sentence from the neighbors above. Left: the parse of the target sentence under the grammar in Section~\ref{sec:WCFG} corresponding to the derivation on the bottom right.}
    \label{fig:example}
\end{figure*}

\section{Background and Notation}
Conditional text generation tasks involve generating a sequence of tokens $\hat{y}_1, \ldots, \hat{y}_T \niceeq \hat{y}_{1:T}$ conditioned on some $x \nicein \mcX$, where each generated token $\hat{y}_t$ is from a vocabulary $\mcV$. We will consider in particular the task of table-to-text generation, where $x$ is some tabular data and $\hat{y}_{1:T}$ is a natural language description of it. 

For supervision, we will assume we have access to a dataset, which pairs an input $x$ with a \textit{true} corresponding reference text $y_{1:T_x} \nicein \mcV^{T_x}$ consisting of $T_x$ tokens. Since we are interested in nearest neighbor-based generation, we will also assume that along with each input $x$ we have a set $\mcN \niceeq \neset$ of $N$ neighbor sequences, with each $\nu^{(n)}_t \nicein \mcV$. We will be interested in learning to form $y_{1:T_x}$ from its corresponding $x$ and neighbor set $\mcN$ in a way that will be made more precise below. We note that finding an appropriate set of neighbor sequences to allow for successful generation with respect to an input $x$ is an interesting and challenging problem (see, e.g., \citet{hashimoto2018retrieve}), but for the purposes of our exposition we will assume these neighbor sequences are easy to obtain given only $x$ (and without knowledge of $y$). We give the details of our simple retrieval approach in Section~\ref{sec:experiments}. 

\subsection{Imitation Learning for Text Generation}
\label{sec:ilbackground}
Much recent work views conditional text generation as implementing a policy $\pi : \mcX \times \mcV^* \rightarrow \mcA \cup \{\stopsymb\}$~\citep{bengio2015scheduled,ranzato2016sequence}; see~\citet{welleck2019non} 
for a recent review. That is, we view a generation algorithm as implementing a policy that consumes an input $x \nicein \mcX$ as well as a partially constructed output in the Kleene closure of $\mcV$, which we will refer to as a ``canvas'' \citep{stern2019insertion}, and which outputs either an action $a \nicein \mcA$ or a decision to stop. Taking action $a$ leads (deterministically in the case of text generation) to a new canvas, and so generation is accomplished by following $\pi$ from some distinguished start canvas until a $\stopsymb$ decision is made, and returning the resulting canvas. For example, sequence-to-sequence style generation~\citep{sutskever2014sequence,cho2014on} implements a policy $\pi$ that consumes $x$ and a canvas $\hat{y}_{1:M} \nicein \mcV^{M}$ representing a prefix, and produces either an action $a \nicein \mcA \niceeq \mcV$, or else a $\stopsymb$ action and generation terminates. When an action $a \nicein \mcV$ is chosen, this leads to the formation of a new prefix canvas $\hat{y}_{1:M} \cdot a$, where $\cdot$ is the concatenation operator.

Imitation learning of text generation policies conventionally proceeds by ``rolling in'' to a canvas $\hat{y}_{1:M}$ using a roll-in policy, and then training a parameterized policy $\pi_{\btheta}$ to mimic the actions of an oracle policy $\pi^*$ run from $\hat{y}_{1:M}$. The most common form of such training in the context of neural text generation is known as ``teacher forcing''~\citep{williams1989learning}, which simply amounts to using $\pi^*$ to roll-in to a canvas $y_{1:M}$, viewing $\pi_{\btheta}$ as a probabilistic classifier, and training $\pi_{\btheta}$ using the action $\pi^*(x, y_{1:M})$ as its target. 

We will adopt this policy-based perspective on conditional text generation, and we will attempt to learn a policy that generates text by splicing together text from retrieved neighbors. In order to do so, we will make use of a more general form of policy than that implemented by standard sequence-to-sequence models. Our policies will allow for both inserting an arbitrary span of neighbor text (rather than a single token) anywhere in a canvas, as well as for \textit{replacing} an arbitrary span of text in a canvas with a span of neighbor text, as we make more precise in the next section. Before doing so, we note that there has been much recent interest in generalizing the forms of policy used in generating text; see Section~\ref{sec:related} for references and a discussion. 

\section{Splicing Nearest Neighbors}
\label{sec:splicing}
 Given a canvas $\hat{y}_{1:M} \nicein \mcV^M$ and a set of neighbor sequences $\mcN \niceeq \neset$, we define a generalized insertion function, which forms a new canvas from $\hat{y}_{1:M}$. This generalized insertion function implements the following mapping

\vspace*{-0.5cm} 
\begin{align}\label{eq:insert}
   \small
    \ginsert(\hat{y}_{1:M}, i, j, n, k, l) \mapsto \hat{y}_{1:i} \cdot \nu^{(n)}_{k:l} \cdot \hat{y}_{j:M},
\end{align}

\vspace*{-0.2cm} 
\noindent where $\cdot$ is again the concatenation operator, the slice indexing is inclusive, $ 0 \leq i < j \leq M+1$,\footnote{We take $\hat{y}_{1:0}$ and $\hat{y}_{M+1:M}$ to be empty sequences.} and $1 \leq k \leq l \leq T_n$. Note that this generalized insertion function allows both for inserting a span into any position in the canvas (when $j \niceeq i+1$), as well as for \textit{replacing} a span anywhere in the canvas with another span (when $j > i+1$), which results in the removal of tokens from the canvas. Intuitively, this generalized insertion function attempts to capture a generation scheme where text is generated by making only minor insertions or replacements in some existing text. For example, we might imagine generating a new sentence by copying a neighbor sentence to our canvas, and then simply replacing the names or dates in this neighbor sentence to form a new sentence; see Figure~\ref{fig:example} for an example.

Having defined this insertion function, we can generate text with a policy that consumes an input $x$, a set of neighbors, and a canvas,\footnote{To conform with our generation policy definition above, we can view $\mcX$ as containing input and neighbor set pairs.} 
and outputs the arguments of the $\ginsert$ function, or else the $\stopsymb$ action. Thus, for a given canvas and neighbor set, we take our policy to produce actions in 

\vspace*{-0.5cm}
{\small
\begin{align*}
\actset(\hat{y}_{1:M}) \niceeq \{&(i, j, n, k, l) \nicein \naturals^5 \mid 0 \leq i < j \leq M{+}1, \\
&1 \leq n \leq N, 1 \leq k \leq l \leq T_n \}
\end{align*}
}

\vspace*{-0.6cm}
\noindent or else the $\stopsymb$ action. We show examples of generating with such a policy in Figure~\ref{fig:example} and Figure~\ref{fig:derivations}.

\subsection{An Oracle Policy}
\label{sec:oraclepolicy}
As described in Section~\ref{sec:ilbackground}, we are interested in learning a parameterized text generation policy $\pi_{\btheta}$. Since we would like to generate using the generalized $\ginsert$ function in Equation~\eqref{eq:insert}, we will attempt to learn a parameterized distribution $\pi_{\btheta}(\cdot \given x, \hat{y}_{1:M}, \mcN)$ over the arguments to this $\ginsert$ function given input $x$, canvas $\hat{y}_{1:M}$, and neighbor set $\mcN$, by training it with the one-hot action distribution $\pi^*(x, y_{1:M}, \mcN)$ as its target. In order to do so, however, we must first obtain an oracle policy $\pi^*$. That is, for each true output $y_{1:T_x}$ in our dataset, we require an oracle sequence of canvases paired with corresponding oracle actions in $\actset$, which derive $y_{1:T_x}$ from $x$ and $\mcN$. In this section we suggest an approach to obtaining these.

In what follows, we will assume that each word type represented in $y_{1:T_x}$ is also represented among the neighbor sequences; in practice, we can always ensure this is the case by using the following expanded neighbor set: $\mcN' \niceeq \neset \cup \mcV$, where the vocabulary $\mcV$ is viewed as containing spans of length one. Thus, policies will be able to emit any word in our vocabulary. 
Furthermore, because the source table $x$ itself will often also contain spans of words that might be used in forming $y_{1:T_x}$, going forward we will also assume $\mcN'$ includes these spans from $x$. 

In arriving at an oracle policy, we first note that given $\mcN'$ there will often be many sequences of actions in $\actset$ that derive the reference text $y_{1:T_x}$ from an empty canvas. For instance, we can simulate standard left-to-right, token-by-token generation with a sequence of actions $((i, i+1, n_i, k_i, k_i))_{i=0}^{T-1}$ such that $\nu^{(n_i)}_{k_i} \niceeq y_i$. However, other derivations, which insert or replace spans at arbitrary canvas locations, will often be available. We posit that derivations with fewer actions will be more interpretable, all else equal, and so we define our oracle policy to be that which derives $y_{1:T_x}$ from $\mcN'$ (starting from an empty canvas) in as few actions as possible. We show this optimization problem can be reduced to finding the lowest-cost parse of $y_{1:T_x}$ using a particular weighted context free grammar (WCFG)~\citep{salomaa1969probabilistic}, which can be done in polynomial time with, for instance, the CKY algorithm~\citep{kasami1966efficient,younger1967recognition,baker1979trainable}. 

\subsubsection{Reduction to WCFG Parsing} \label{sec:WCFG}
Consider the following WCFG in Chomsky Normal Form~\citep{chomsky1959certain}:


\vspace*{-0.3cm}
\begin{alignat*}{3} 
    [1] \quad &S &&\rightarrow \nu^{(n)}_{k:l} \quad &&\forall n, k \leq l \\
    [0] \quad&S &&\rightarrow S \, S \quad && \\
    [1] \quad&S &&\rightarrow \nu^{(n)}_{k:l} \, C^{(n)}_s \quad &&\forall n, l < s\\
    [0] \quad&C^{(n)}_s &&\rightarrow S \, R^{(n)}_s \quad &&\forall n, s \\
    [0] \quad&R^{(n)}_s &&\rightarrow \nu^{(n)}_{s:t} \quad &&\forall n, s \leq t \\
    [0] \quad &R^{(n)}_s &&\rightarrow \nu^{(n)}_{s:t} \, C^{(n)}_{u} \quad &&\forall n, s \leq t < u,
\end{alignat*}
where $S$ is the start non-terminal 
and where the bracketed number gives the cost of the rule application. 
We can see that there is a cost (of 1) for introducing a new neighbor span with an $S$ non-terminal, but no cost for continuing with the remainder of a neighbor already introduced (as represented by the $C$ and $R$ non-terminals).

\begin{claim} \label{claim1}
Given neighbors $\neset$, the length of the shortest derivation of a sequence $y_{1:T_x}$ using actions in $\actset$ is equal to its lowest cost derivation under the WCFG above.
\end{claim}

\noindent We prove the above claim in \cref{sec:proof}.
The proof proceeds by two simulation arguments.
First, we show that, given a derivation of $y_{1:T_x}$ with a certain weight under the WCFG, we can simulate a subset of the derivations with a number of $\ginsert$ operations equal to the total weight of the derivations and still obtain $y_{1:T_x}$. This implies that the cost of the \emph{optimal} derivation under the WCFG is at least the cost of the \emph{optimal} number of $\ginsert$ operations.
Second, we show that, given a derivation of $y_{1:T_x}$ with a certain number of $\ginsert$ operations, we can simulate a subset of these operations with a cost-$1$ derivation of the grammar per $\ginsert$ operation. This implies that the optimal number of $\ginsert$ operations is at least the cost of the \emph{optimal} derivation according to the WCFG. Together, the two simulation arguments imply the claim. 

\paragraph{Complexity} If $T$ is the maximum length of all sequences in $\{y_{1:T_x}\} \cup \mcN'$ and $|\mcN'| \niceeq N$, parsing under the above WCFG with the CKY algorithm is $O(NT^6)$. The runtime is dominated by matching the $S \rightarrow Y^{(n)}_{k:l} \, C^{(n)}_s$ rule; there are $O(NT^3)$ sequences that match the right-hand side (all $k \leq l < s$ for all $\nu^{(n)}$), and we must consider this rule for each span in $y_{1:T_x}$ and each split-point.

\paragraph{Obtaining the policy} Using Claim~\ref{claim1}, we obtain an oracle action sequence deriving $y_{1:T_x}$ from its neighbors $\mcN'$ by first computing the minimum-cost parse tree. As noted in Section~\ref{sec:oraclepolicy}, $\mcN'$ is guaranteed to contain any word-type in $y_{1:T_x}$. In practice, we ensure this by only adding word-types to $\mcN$ that are not already represented in some neighbor, so that computed oracle parses use the neighbor sequences rather than the vocabulary. Given the minimum-cost parse tree, we then obtain a sequence of $\ginsert$ actions by doing a depth-first left-to-right traversal of the tree.\footnote{While there is a derivation corresponding to \textit{each} depth-first traversal, the left-to-right traversal performed best.} In particular, we can obtain all the arguments for an $\ginsert$ operation after seeing all the children of its corresponding $S$ non-terminal. For example, in Figure~\ref{fig:example}, the insert operations on the bottom right follow the order in which $S$ non-terminals are encountered in a left-to-right, depth-first traversal of the tree on the left; the arguments of the operation that introduces $\nu^{(2)}$, for example, are determined by keeping track of the corresponding $S$'s distance from the left sentence-boundary and the length of the span it yields. We precompute these oracle derivations for each $(x, y_{1:T_x}, \mcN')$ triplet in our training corpus. 

\subsection{Additional Oracle Policies}
We will refer to policies derived as above as ``\gimodel{}'' policies. While \gimodel{} policies minimize the number of $\ginsert$ operations used in deriving $y_{1:T_x}$, there are at least two other reasonable neighbor-based oracle policies that suggest themselves. One is the oracle policy that derives $y_{1:T_x}$ from left to right, one token at a time. 
This policy is identical to that used in training sequence-to-sequence models, except each token comes from $\mcN'$. In particular, a generated token is always copied from a neighbor sequence if it can be. We will refer to this policy as ``\ltwort,'' for ``left-to-right, token-level.''

Another oracle policy one might consider would allow for inserting spans rather than words left-to-right, but like \gimodel{} policies would attempt to minimize the number of span insertion operations. While a greedy algorithm is sufficient for deriving such policies, in preliminary experiments we found them to consistently underperform both \gimodel{} and \ltwort{}, and so we do not consider them further.



\section{Models, Training, and Generation}
To avoid directly parameterizing a distribution over the impractically large number of combinations of arguments to the $\ginsert$ function, we factorize the distribution over its arguments as

\vspace*{-0.3cm}
{
\begin{align} \label{eq:factorization}
\pi_{\btheta}(i, j, n, k, l) &= \pi_{\btheta}(j, l \given i, n, k) \times \pi_{\btheta}(i, n, k),
\end{align}
}

\vspace*{-0.3cm}
\noindent where we have left out the explicit conditioning on $x$, $\hat{y}_{1:M}$, and $\mcN'$ for brevity. Thus, our policy first predicts an insertion of token $\nu^{(n)}_k$ after the canvas token $\hat{y}_i$. Conditioned on this, the policy then predicts the final token $\nu^{(n)}_l$ of the inserted span, and which canvas token $\hat{y}_j$ immediately follows it.

More concretely, we obtain token-level representations $\boldx_1, \ldots, \boldx_S$ and $\hat{\boldy}_0, \ldots, \hat{\boldy}_{M+1}$, all in $\reals^d$, of source sequence $x \niceeq x_{1:S}$ and of canvas sequence $\hat{y}_{1:M}$, padded on each side with a special token, by feeding them to an encoder-decoder style transformer~\citep{vaswani2017attention} with no causal masking. We obtain neighbor token representations $\boldsymbol{\nu}^{(n)}_k$ by feeding neighbor sequences through the same encoder transformer that consumes $x$. We provide additional architectural details in Appendix~\ref{sec:hypers}. Viewing the source sequence $x$ as the $0$th neighbor, we then define

\vspace*{-0.4cm}
{ 
\begin{align*}
    \pi_{\btheta}(i, n, k) \propto \begin{cases} \exp(\hat{\boldy}_{i-1}^{\top} \boldW_1 \boldW_0 \, \boldx_k) &\text{if $n \niceeq 0$} \\
                                                 \exp(\hat{\boldy}_{i-1}^{\top} \boldW_1 \boldW_2 \, \boldsymbol{\nu}^{(n)}_k) &\text{if $n \, {>} \, 0$,}
                                    \end{cases}
\end{align*}
}

\vspace*{-0.2cm}
\noindent where the normalization is over all pairings of a canvas token with either a neighbor or source token, and where $\boldW_0$, $\boldW_1$, and $\boldW_2$, all in $\reals^{d\times d}$, are learnable transformations. Similarly, we let

\vspace*{-0.6cm}
{ 
\begin{align*}
    \pi_{\btheta}(j, l &\given i, n, k) \propto \\
    &\begin{cases} \exp(\hat{\boldy}_{j+1}^{\top} \boldW_4 \boldW_3 \, \boldx_l) &\text{if $n \niceeq 0$} \\
                                                 \exp(\hat{\boldy}_{j+1}^{\top} \boldW_4 \boldW_5 \, \boldsymbol{\nu}^{(n)}_l) &\text{if $n \, {>} \, 0$,}
                                    \end{cases}
\end{align*}
}

\vspace*{-0.2cm}
\noindent where now the normalization only considers pairings of the $j$th canvas token with the $l$th neighbor or source token, where $j > i, l \geq k$. $\boldW_3$, $\boldW_4$, and $\boldW_5$ are again learnable and in $\reals^{d\times d}$.

While the above holds for \gimodel{} policies, \ltwort{} policies always insert after the most recently inserted token, and so they require only a $\pi_{\btheta}(n, k \given i)$ policy, and only $\boldW_0$ and $\boldW_2$ transformations.

\subsection{Training}
\label{sec:training}
As noted in Section~\ref{sec:oraclepolicy}, we propose to train our policies to imitate the oracle derivations obtained by the CKY parse, using teacher-forcing. Suppose that, for a given (oracle) canvas $y_{1:M}$ and set of neighbors $\mcN'$, the oracle next-action obtained from the parse is $(i^*, j^*, n^*, k^*, l^*)$. Since there may be multiple spans in $\mcN'$ that are identical to $\nu^{(n^*)}_{k^*:l^*}$, we train the $\pi_{\btheta}(i, n, k)$ policy to minimize

\vspace*{-0.3cm}
{ 
\begin{align} \label{eq:loss1}
    - \log \sum_{\{(n, k, l) \mid \nu^{(n)}_{k:l} = \nu^{(n^*)}_{k^*:l^*} \}} \pi_{\btheta}(i^*, n, k).
\end{align}
}

\vspace*{-0.3cm}
\nocite{he2012imitation} 
The $\pi_{\btheta}(j, l \given i, n, k)$ policy is simply trained to minimize

\vspace*{-0.4cm}
{ 
\begin{align} \label{eq:loss2}
    - \log \pi_{\btheta}(j^*, l^* \given i^*, n^*, k^*),
\end{align}
}

\vspace*{-0.4cm}
\noindent since there is one correct target given $i^*, n^*, k^*$.

Training proceeds by sampling a mini-batch of examples and their derivations, and minimizing the sums of the losses~\eqref{eq:loss1} and~\eqref{eq:loss2} over each action in each derivation, divided by the mini-batch size. 


\subsection{Generation}
\label{sec:generation}
For \gimodel{} models, we generate with beam search, following the factorization in Equation~\ref{eq:factorization}. At each iteration, the beam first contains the top-$K$ partial hypotheses that can be constructed by predicting the $i, n, k$ arguments to the $\ginsert$ function given the current canvas and neighbors. Given these, the remaining $j, l$ arguments are predicted, and the top $K$ of these are kept for the next iteration. We search up to a maximum number of actions, and in computing the final score of a hypothesis, we average the $\pi_{\btheta}$ log probabilities over all the actions taken to construct the hypothesis (rather than summing). 

For \ltwort{} models, we generate with standard left-to-right, token-level beam search. We note that in this setting it is common to marginalize over all occurrences of a word-type (e.g., among neighbors or in the table) in calculating its probability. While this generally improves performance (see below), it also hurts interpretability, since it is no longer clear which precise neighbor or source token gives rise to a predicted token. Below we report results in both the standard marginalization setting, and in a no-marginalization (``no-marg'') setting.



\section{Experiments}
\label{sec:experiments}
Our experiments are designed to test the quality of the text produced under \gimodel{} policies and \ltwort{} policies, as well as whether such policies allow for more controllable or interpretable generation.

\paragraph{Datasets} We expect our approach to work best for tasks where different generations commonly share surface characteristics. Table-to-text tasks meet this requirement, and are accordingly often used to evaluate generation that makes use of retrieved neighbors~\citep{peng2019text,lin2020record} or induced templates~\citep{wiseman2018learning,li2020posterior}. Following recent work, we evaluate on the E2E~\citep{novikova2017e2e} and WikiBio~\citep{lebret2016neural}  datasets. 

\paragraph{Preprocessing}
We whitespace-tokenize the text, and mask spans in neighbor sequences that appear in their corresponding sources, which discourages derivations from copying content words from neighbors. We pad each $y$ and $\nu^{(n)}$ sequence with beginning- and end-of-sequence tokens, which encourages derivations that insert into the middle of sequences, rather than merely concatenating spans. 

\paragraph{Obtaining Neighbors} We precompute neighbors for each training example, taking the top-scoring 20 neighbors for each example in the training data (excluding itself) under a simple score $s(\cdot, \cdot)$ defined over pairs of inputs in $\mcX$. For the E2E and WikiBio datasets, we define $s(x, x') \niceeq F_1(\mathrm{fields}(x), \mathrm{fields}(x')) + 0.1 F_1(\mathrm{values}(x),\mathrm{values}(x'))$, where $\mathrm{fields}$ extracts the field-types (e.g., ``name'') from the table $x$, $\mathrm{values}$ extracts the unigrams that appear as values in $x$, and $F_1$ is the $F_1$-score. 

\paragraph{Baselines}
We compare \gimodel{} policies to \ltwort{} policies, to a transformer-based sequence-to-sequence model with a copy mechanism~\citep{gu2016incorporating} that uses no retrieved neighbors (henceforth ``S2S+copy''), and to recent models from the literature (see below). The S2S+copy model uses a generation vocabulary limited to the 30k most frequent target words. The neighbor-based policies, on the other hand, are limited to generating (rather than copying) only from a much smaller vocabulary consisting of target words that occur at least 50 times in the training set and which cannot be obtained from the target's corresponding neighbors. 

\paragraph{Additional Details}  All models are implemented using 6-layer transformer encoders and decoders, with model dimension 420, 7 attention heads, and feed-forward dimension 650;\footnote{We use the \texttt{huggingface}~\citep{wolf2020transformers} implementation of the BART~\citep{lewis2020bart} architecture, but with no pretraining.} all models are trained from scratch. We train with Adam~\citep{kingma2015adam,loshchilov2018decoupled}, using linear learning-rate warm-up and square-root decay as in \citet{devlin2019bert}, until validation loss stops decreasing. We generate with beam search (see Section~\ref{sec:generation}), and neighbor-based models use 20 neighbors at test time, just as at training time. We discuss hyperparameters and tuning in Appendix~\ref{sec:hypers}. We include sample generations from all systems in Appendix~\ref{sec:samplegens}, and additional visualizations of some \gimodel{} generations in Figure~\ref{fig:derivations}.

\begin{table}[t!]
\small
\centering
\begin{tabular}{@{}lccccc@{}}
\toprule
\textbf{E2E} & BLEU & NIST & RG & CID & MET  \\
\midrule
\gimodel & 70.5 & 9.54 & 76.0 & 2.37 & 49.6 \\
\ltwort-no-marg & 55.8 & 7.39  & 63.7 & 1.68 & 41.0 \\
\ltwort & 68.1 & 8.83 & 70.0 & 2.38 & 46.2 \\
\midrule
S2S+copy & 64.7 & 8.26 & 69.1 & 2.22 & 43.7 \\
Li \& Rush & 67.1 & 8.52 & 68.7 & 2.24 & 45.4 \\
KGPT & 68.1 & -    & 70.9 & - & 45.8 \\
\bottomrule
\toprule
\textbf{WB} & BLEU & NIST & RG-4 & & \\ 
\midrule
\gimodel & 43.5 & 9.59 & 41.4 &  &  \\
\ltwort-no-marg &  45.4  &  10.16   &  39.6 & & \\
\ltwort & 45.7 & 9.99 & 44.0 &  &  \\
\midrule
S2S+copy  & 45.4 & 9.72 & 44.6 & & \\
Peng et al. & 44.1 & - & 41.1 &  &  \\
Li \& Rush   & 44.7 & 9.92 & 43.3 & \\
KGPT   & 45.1 & -    &  - & \\
\bottomrule
\end{tabular}
\caption{Standard automatic evaluation metrics for the E2E dataset (top) and WikiBio dataset (bottom). Baselines include our own transformer sequence-to-sequence-with-copy model (``S2S+copy''), and the models of \citet{li2020posterior}, \citet{peng2019text}, and \citet[``KGPT'']{chen20kgpt}.}
\label{tab:autoresults}
\end{table}


\begin{table}[t!]
\small
\centering
\begin{tabular}{lccc}
\toprule
\textbf{E2E} & Natural & Faithful & Informative \\
\midrule
\gimodel & 3.87 & 3.97 & 3.89\\
\ltwort  & 3.75 & 3.94 & 3.94 \\
S2S+copy & 3.87 & 3.94 & 3.81 \\
\bottomrule
\toprule
\textbf{WB} &  &  &\\ 
\midrule
\gimodel & 3.69 & 3.52   & 3.27 \\
\ltwort  & 3.83 & 3.74   & 3.37 \\
S2S+copy & 3.75 & 3.75   & 3.40 \\
\bottomrule
\end{tabular}
\caption{Average rating (on 1-5 Likert scale) of generations' naturalness, faithfulness, and informativeness, according to crowd-workers. No pairwise differences are significant under a Tukey HSD test.}
\label{tab:humanresults}
\end{table}

\subsection{Quality Evaluation}
We first evaluate our models and baselines using the standard automatic metrics associated with each dataset, including BLEU~\citep{papineni02bleu}, NIST, ROUGE~\citep{lin2004rouge}, CIDEr~\citep{vedantam2015cider} and METEOR~\citep{banerjee2005meteor}, in Table~\ref{tab:autoresults}. There we also compare  with the model of ~\citet{peng2019text}, which uses retrieved neighbors, and of~\citet{li2020posterior}, which produces interpretable segmentations, as well as with the model of \citet{chen20kgpt} (``KGPT'' in tables), which is a fine-tuned, large pretrained model, and which we take to be close to the state of the art.

We first note that our baselines are quite strong, largely outperforming previous work, including large pretrained models. In the case of E2E, we find that the \gimodel{} model slightly outperforms these strong baselines and attains, we believe, state-of-the-art performance in the setting where no pretrained models or data augmentation is used. (See \citet{chang2021neural} for even better results without these restrictions). On WikiBio, however, \gimodel{} slightly underperforms the strongest baselines.



\paragraph{Human Evaluation} In Table~\ref{tab:humanresults} we show the (average) results of a human evaluation conducted following the methodology described in \citet{reiterblog}. We ask crowd-workers on Amazon Mechanical Turk to score generations in terms of their naturalness, their faithfulness to the source table, and their informativeness, on a 5-point Likert scale. (Note that following \citet{dhingra2019handling} we ask about informativeness rather than usefulness). We score a total of 45 random examples from each test dataset, with each generation being rated by 3 crowd-workers, and each crowd-worker seeing a generation from each system. 
We ran multi-way ANOVAs with system-type (i.e., \gimodel{}, \ltwort{}, or S2S+copy), example index, and crowd-worker-id as independent variables, and the rating as the dependent variable. 

The only significant interaction involving system-type was with respect to ``faithfulness'' on the WikiBio dataset ($p < 0.018$), though this does \textit{not} reflect a necessary correction accounting for the multiple comparisons implied by crowd-workers rating along 3 dimensions. Furthermore, under Tukey's HSD test no significant pairwise (i.e., between system pairs) interactions were found in any setting. Thus, we find no significant difference between any system pairs according to crowd-workers, although (as with the automatic metrics) \gimodel{} performs slightly better on E2E and worse on WikiBio.
We give the precise $p$-values as well as more details about the questions crowd-workers were asked in Appendix~\ref{sec:heval}. 

We also conduct a manual analysis of the faithfulness errors made by \gimodel{} generations in Appendix~\ref{sec:wikierrors}; we generally find that \gimodel{} generations do not hallucinate more than S2S+Copy generations (the most faithful generations according to crowd-workers), but they do more frequently contradict information in the source table. This generally occurs when a span containing information that contradicts the table is copied to the canvas, and this information is not subsequently replaced; see Appendix~\ref{sec:wikierrors} for more details and a discussion.


\subsection{Interpretability Evaluation}
\label{sec:interp}
We first emphasize that, on an intuitive level, we believe \gimodel{} policies lead to significantly more interpretable generation than token-level policies. This is because \gimodel{} policies give an explicit (and often short) span-based derivation of a generated text in terms of neighbor text. We show visualizations of two randomly chosen generations from the WikiBio validation set, along with their derivations, in Figure~\ref{fig:derivations}. 


However, it is difficult to precisely quantify the interpretability of a text generation model, and so we now quantify several aspects of our models' predictions that presumably correlate with their interpretability. Table~\ref{tab:interp} shows the average length of the derivations (i.e., how many $\ginsert$ operations are required to form a generation), the average number of neighbors used in forming a prediction, and the percentage of generated tokens copied from a neighbor or the source $x$ (rather than generated from the model's output vocabulary) for \gimodel{} and \ltwort{}-no-marg policies over 500 randomly-chosen examples from the E2E and WikiBio validation-sets. All else equal, we expect fewer $\ginsert$ operations, fewer neighbors, and more tokens copied from neighbors (resp.) to correlate with better interpretability. We find that \gimodel{} generations require many fewer $\ginsert$ operations and distinct neighbors per generation on average than \ltwort{} generations, although they use their output-vocabulary slightly more than \ltwort{}-no-marg. Note that we use \ltwort{}-no-marg for this comparison because marginalization obscures whether a predicted token is from a neighbor.

\begin{table}[t!]
\small
\centering
\begin{tabular}{lcc}
\toprule
                & \gimodel{}         & \ltwort{}-no-marg \\ 
                & E2E / WB           & E2E / WB         \\ 
\midrule
\# Inserts      &     6.9 / 7.4      &   24.7 / 25.8     \\ 
\# Neighbors    &     1.1 / 1.4    &  5.5  / 5.6     \\ 
\% Tok. Copied  &     99.3 / 95.3    &   99.8  / 96.4     \\ 
\bottomrule
\end{tabular}
\caption{Number of inserts, number of neighbors used, and percentage of generation tokens from a neighbor, averaged over 500 random examples from the E2E and Wikibio validation sets.} 
\label{tab:interp}
\end{table}


\begin{figure}[t!]
    \centering
    \hspace*{-0.5cm}
    \includegraphics[scale=0.39]{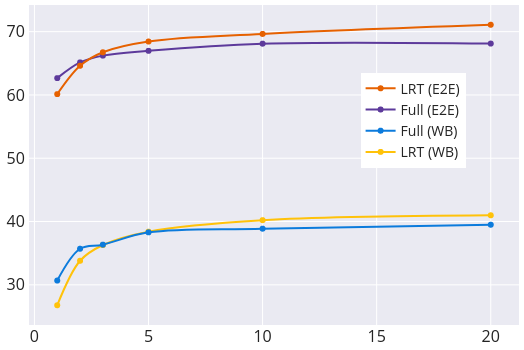}
    \caption{ROUGE validation performance on WikiBio and E2E, by number of neighbors used at test time.}
    \label{fig:neighborcurve}
\end{figure}

The fact that \gimodel{} policies use so few distinct neighbors per example motivates asking how well these policies perform at test time when given fewer neighbors than they are trained with (namely, 20 in all experiments). We plot the average validation ROUGE of \gimodel{} and \ltwort{} for both datasets (using ROUGE-4 for WikiBio and ROUGE-L for E2E, as is conventional) against the number of neighbors used at generation time in Figure~\ref{fig:neighborcurve}. We see that while using fewer neighbors hurts both types of policies, \gimodel{} outperforms \ltwort{} for very few neighbors.

\paragraph{Controllability} Another approach to evaluating the interpretablility of a model is to use our understanding of the model's prediction process to control it, and then evaluate controllability. In Appendix~\ref{sec:casestudy} we describe, as a case study, attempting to control the number of sentences used in E2E dataset generations by controlling the neighbors; we find that \gimodel{} significantly outperforms \ltwort{} policies in ensuring that generations have at least three sentences.

\begin{figure*}[t!]
    \centering
    \includegraphics[scale=0.38]{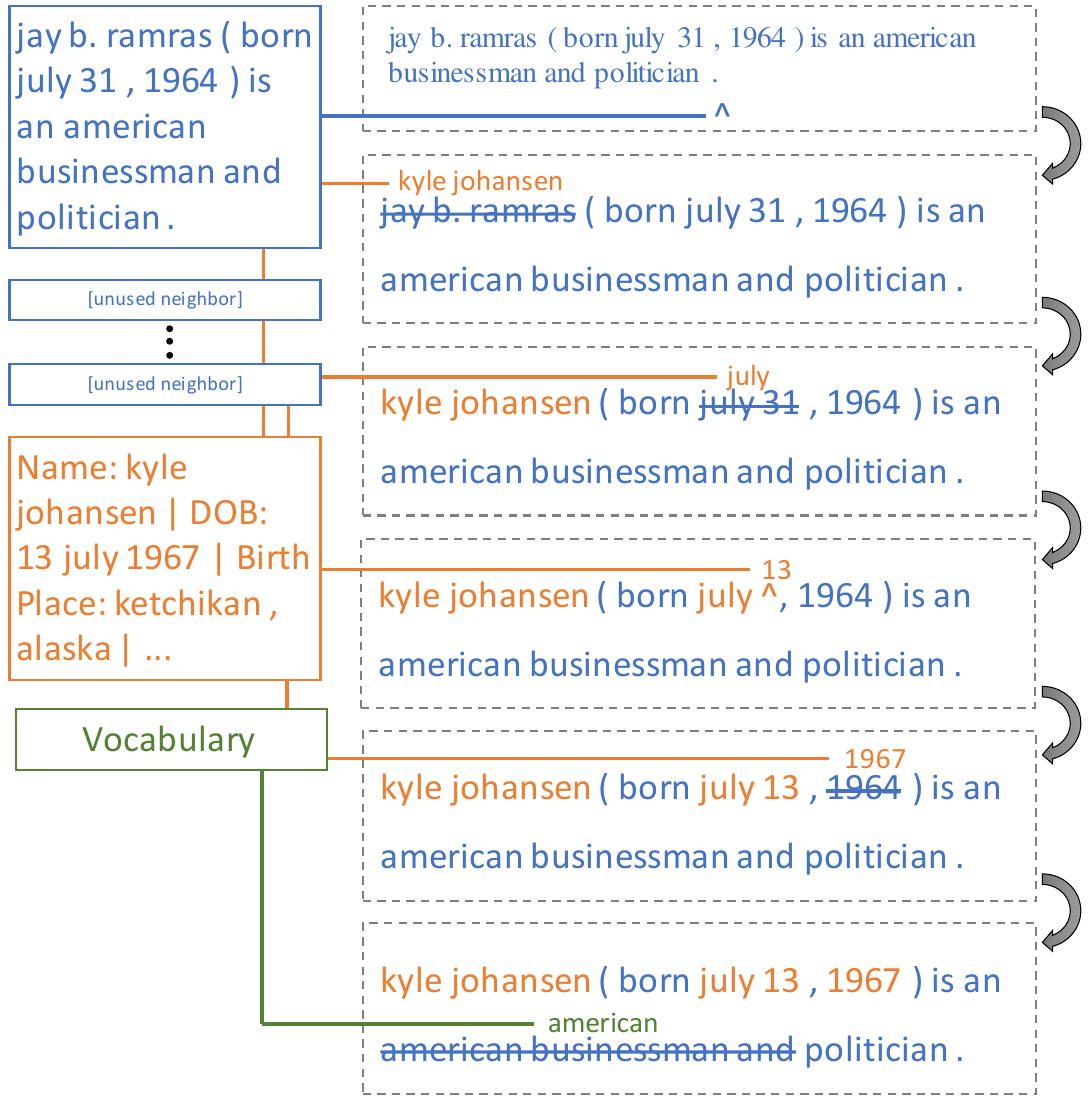}
    \hspace*{0.2cm}
    \includegraphics[scale=0.5]{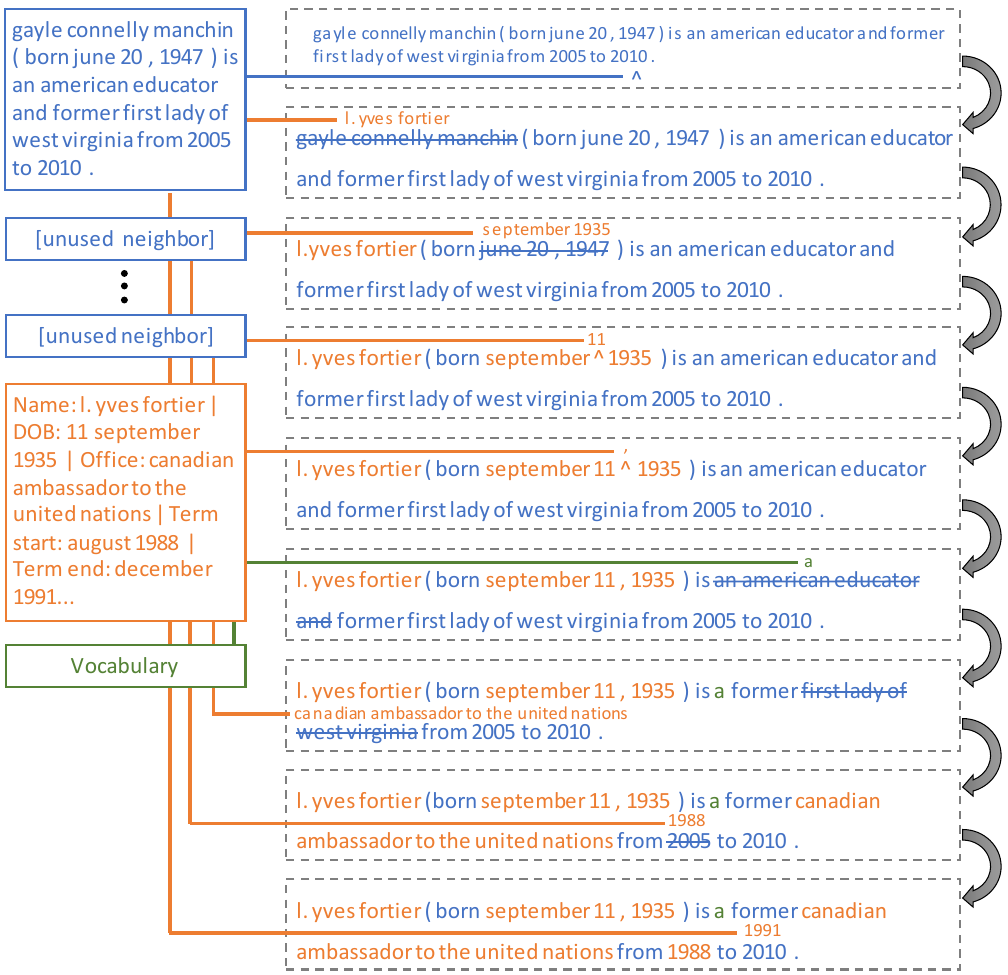}
    \caption{A visualization of how two randomly selected model generations from the WikiBio validation set --- \textit{kyle johansen (born july 13, 1967) is an american politician.} and \textit{l. yves fortier (born september 11, 1935) is a former canadian ambassador the the united nations from 1988 to 1991.} --- are derived. In both cases only one non-table neighbor is used, and is depicted in the top-most solid box (in blue). The solid box fourth from the top (in orange) represents the linearized source table, $\nu^{(0)}$, and the fifth (in green) represents all word types in the vocabulary. Derivations should be read top-down; the current canvas is represented by a dotted box, carets indicate insertion, struck through text has been replaced, and colored lines indicate the provenance of a span. Beginning- and end-of-sequence tokens are omitted.}
    \label{fig:derivations}
\end{figure*}





\section{Related Work}
\label{sec:related}
NLP systems have incorporated neighbors for decades. Early work focused on machine translation~\citep{sumita1991experiments}, syntactic disambiguation~\citep{cardie1994domain}, and tagging~\citep{daelemans1993memory,daelemans1996mbt}.

While some more recent work has made use of retrieved neighbors for problems such as sequence labeling~\citep{wiseman2019label}, auditing multi-label text classification predictions~\citep{schmaltz2020exemplar}, and reasoning over knowledge bases~\citep{das2020simple,das2021case}, the majority of recent NLP work involving neighbor-based methods has focused on conditioning neural text generation systems on retrieved neighbors. This conditioning is variously accomplished using a conventional encoder in an encoder-decoder setup~\citep{song2016two,weston2018retrieve,gu2018search,cao2018encoding,bapna2019non}, by allowing the parameters of the decoder to depend on the retrieved neighbor~\citep{peng2019text}, or by viewing the unknown neighbor as a latent variable~\citep{hashimoto2018retrieve,guu2018generating,chen2019controllable,he2020learning}. Recent work~\citep{zhang2018guiding,khandelwal2019generalization,khandelwal2020nearest} has also used retrieved neighbors at decoding time to modify the next-token distribution of the decoder. Our work differs from these approaches in that we explicitly parameterize the splicing operations that form a generation from neighbors, rather than conditioning or otherwise modifying a left-to-right token generation model using retrieved neighbors.



Our parameterization is motivated by trying
to increase the interpretability and controllability of the generation process, which also motivates recent work making explicit the template or plan being followed by the generation~\citep[\textit{inter alia}]{iyyer2018adversarial,wiseman2018learning,puduppully2019data,chen2019controllable,li2020posterior}. 
This more structural or syntactic flavor of controllability differs slightly from foundational work on controlling content or stylistic attributes of text~\citep{hu2017controllable,ficler2017controlling,fan2018controllable}. 

Our approach is also related to work in non-left-to-right text generation, including tree-based~\citep{welleck2019non,akoury2019syntactically}, non-autoregressive~\citep[\textit{inter alia}]{gu2018non,lee2018deterministic}, masked language model-based~\citep[\textit{inter alia}]{ghazvininejad2019mask}, and, most closely, insertion-based~\citep[\textit{inter alia}]{stern2019insertion,gu2019levenshtein,gu2019insertion} approaches. Our work differs from this last category in several important respects: first, we insert and replace (and model) full spans rather than tokens. Our policies are trained to minimize the number of insertion operations rather than to insert (centrally positioned) correct tokens in available slots, as is Insertion Transformer~\citep{stern2019insertion}, or to mimic a Levenshtein distance-based oracle, as is LevT~\citep{gu2019levenshtein}. Our policies are also fundamentally sequential, unlike these partially autoregressive alternatives, which can generate tokens in parallel. The sequential nature of our approach makes using beam search straightforward (unlike in token-parallel approaches) and, we think, leads to interpretable, serial derivations. On the other hand, decoding serially with beam search will generally be slower than the iterated parallel decoding of partially autoregressive models.

Our work also relates to recent work on sentence-level transduction tasks, like grammatical error correction (GEC), which allows for directly predicting certain span-level edits~\citep{stahlberg2020seq2edits}. These edits are different from our insertion operations, requiring token-level operations except when copying from the source sentence, and are obtained, following a long line of work in GEC~\citep{swanson2012correction,xue2014improved,felice2016automatic,bryant2017automatic}, by heuristically merging token-level alignments obtained with a Damerau-Levenshtein-style algorithm~\citep{brill2000improved}.






\section{Conclusion}
We have presented an approach to data-to-text generation, which directly splices together retrieved neighbors. We believe this line of work holds promise for improved interpretability and controllability of text generation systems. 

In future work we hope to tackle more ambitious text generation tasks, which will likely require retrieving many more neighbors, perhaps dynamically, from larger data-stores, and with more sophisticated retrieval techniques, such as those currently being used in retrieval-based pretraining~\citep{lewis2020pre,guu2020realm}.

We also hope to consider more sophisticated models, which explicitly capture the history of produced canvases, and more sophisticated training approaches, which search for optimal insertions while training, rather than as a preprocessing step~\citep{daume09search,ross11a}.





\bibliography{custom}
\bibliographystyle{acl_natbib}

\appendix

\newpage
\newpage
\ \ \ 
\\\\ 
\\\\ \\\\\\
\\\\\\
\\\\\\
\\\\\
\\\\\
\\\\\
\\\\
\\\\\
\\\\\
\\\\\
\\\\
\\\\\
\\\\\
\\\\\
\\\\
\\\\\
\\\\
\\\\\
\\\\
\\\\\
\\\\

\section{Proof of Claim \ref{claim1}}
\label{sec:proof}

Given sequences $\nu^{(1)}, \ldots, \nu^{(N)}$ and a target sequence $y$,
let $C$ be the minimum integer such that there exist sequences $y^{(0)}, \ldots, y^{(C)}$ with $y^{(0)} \niceeq \eps$ being the empty sequence, $y^{(C)} \niceeq y$, and $y^{(c)} \niceeq \ginsert(y^{(c-1)},i,j,n,k,l)$
for $c \niceeq 1, \ldots, C$, where $M \niceeq |y^{(c-1)}|$, $0 \, {\leq} \, i \, {<} \, j \, {\leq} \, M{+}1$, and $1 \, {\leq} \, k \, {\leq} \, l \, {\leq} \, |\nu^{(n)}|$. Let $\text{cost}_{\text{ins}}(y)$ be this minimum $C$, equivalent to the length of the shortest derivation of $y$ with actions in $\actset$, and let $\text{cost}_{\text{CFG}}(y)$ be the cost of the minimum cost derivation of $y$ with the WCFG in Section~\ref{sec:WCFG}. We want to show that $\text{cost}_{\text{ins}}(y) \niceeq \text{cost}_{\text{CFG}}(y)$, which we accomplish by showing that $\text{cost}_{\text{ins}}(y)\leq \text{cost}_{\text{CFG}}(y)$ and that $\text{cost}_{\text{ins}}(y)\geq \text{cost}_{\text{CFG}}(y)$. 

\begin{proposition} \label{t2}
	$\text{cost}_{\text{ins}}(y)\leq \text{cost}_{\text{CFG}}(y).$
\end{proposition}

\vspace*{-0.4cm}
\begin{proof}
	Consider the minimum cost derivation tree of $y$ under the WCFG in Section~\ref{sec:WCFG}, and let $C$ be the minimum cost. We show inductively that there exists a sequence of $\leq C$ insert operations that yields $y$ from $\eps$, by considering two cases.
	
	\paragraph{Case 1:} $y$ is derived using only the first two grammar rules, and so is a concatenation of sequences $\nu^{(n)}_{k:l}$ derived from the first grammar rule. If there are $C$ such sequences, the WCFG derivation costs $C$. Also, constructing this sequence using insertions requires at most $C$ insertions. 
	
	\paragraph{Case 2:} $y$ is derived using at least one application of the third grammar rule and can be written as $y \niceeq s^{(1)} \cdot y^{(1)} \cdot s^{(2)} \cdot y^{(2)} \cdot \ldots \cdot y^{(Q)} \cdot s^{(Q+1)}$, where the $s^{(q)}$ sequences are all from the same $\nu^{(n)}$ (using the last three grammar rules), and the $y^{(q)}$ sequences are derived from an $S$ nonterminal. We can derive $y$ using $\ginsert$ operations by inserting a substring of $\nu^{(n)}$ containing all the $s^{(q)}$ (which costs $1$ under the grammar), then inserting the remaining $y^{(q)}$ recursively, which, by induction, costs at most $\text{cost}_{\text{CFG}}(y^{(q)})$. The total number of insertions is then at most $\text{cost}_{\text{CFG}}(y)$.
\end{proof}

\begin{proposition} \label{t1}
	$\text{cost}_{\text{ins}}(y)\geq \text{cost}_{\text{CFG}}(y).$
\end{proposition}

\vspace*{-0.2cm}
\begin{proof}
	Let $\tilde y^{(0)} \niceeq \eps, \tilde y^{(1)}, \ldots, \tilde y^{(C)}$ be sequences of integers defined as follows:
	
	\vspace*{-0.5cm}
	{\small 
	\begin{align*}
		\tilde y^{(c)} = \tilde y^{(c-1)}_{1:i} , \underbrace{c , c \ldots c , c}_{|\nu^{(n)}_{k:l}|}, \tilde y^{(c-1)}_{j:M}.
	\end{align*}
	}
	
	\vspace*{-0.4cm}
	\noindent That is, instead of inserting sequence $\nu^{(n)}_{k:l}$, we insert a sequence of $|\nu^{(n)}_{k:l}|$ integers $c$. We call a sequence $z$ of integers non-interleaving if there is no $i<j<k<l$ such that $z_i = z_k$ and $z_j = z_l$.
	\begin{observation} \label{l1}
		The sequences $\tilde y^{(0)}, \ldots, \tilde y^{(C)}$ are non-interleaving.
	\end{observation}
	
	\vspace*{-0.4cm}
	\begin{proof}
		By induction: if $\tilde y^{(c)}$ is non-interleaving, then so is $\tilde y^{(c+1)}$ by its construction from $\tilde y^{(c)}$.
	\end{proof}
	
	Now let $y' \niceeq \tilde y^{(C)}$ to simplify notation. Let $\text{distinct}(y')\leq C$ be the number of distinct integers in the sequence $y'$. We show how to derive $y$ from the grammar with derivation cost $\text{distinct}(y')$, which proves the proposition. We call an integer $c'$ contiguous in $y'$ if $y'$ can be written as $y' = y'', c', c', \ldots, c', c', y'''$ such that sequences $y''$ and $y'''$ do not contain $c'$. We derive $y$ from the grammar inductively, by considering two cases.
	
	\paragraph{Case 1:} all integers in $y'$ are contiguous, so $y'$ consists of contiguous blocks of repeated integers. Let $b$ be the number of blocks. We invoke the second grammar rule $b\,{-}\,1$ times to get $S$ repeated $b$ times and then invoke the first grammar rule for each $S$ to derive the contiguous sequence $\nu^{(n)}_{k:l}$ corresponding to the block. This costs $\text{distinct}(y') \niceeq b$ in total as required.
	
	\paragraph{Case 2:} there is an integer in $y'$ that is not contiguous. Let $c'$ be the left-most non-contiguous integer in $y'$. $c'$ splits $y'$ into several shorter sequences $y'^{(1)}, \ldots, y'^{(Q)}$, where each sequence $y'^{(q)}$ does not contain any copy of integer $c'$. Since $y'$ is non-interleaving (by Observation \ref{l1}), $y'^{(q)}$ and $y'^{(q')}$ do not share any integers for $q \neq q'$. Therefore,
	$
		\text{distinct}(y') \niceeq 1 + \text{distinct}(y'^{(1)})+ \ldots + \text{distinct}(y'^{(Q)})
	$.
	Furthermore, each sequence $y'^{(q)}$ is non-interleaving. Therefore, we can derive the subsequence of $\tilde y^{(C)}$ corresponding to $y'^{(q)}$ from the non-terminal $S$ and it costs $\text{distinct}(y'^{(q)})$ by induction. To finish the proof we need to show that we can combine the resulting sequences into the sequence corresponding to $y'$ by paying an additional cost of only $1$. We can do that by using the last three rules of the grammar where the rule of cost $1$ is applied only once. In particular, we pay $1$ to derive the sequence corresponding to the first block of integers $c'$ in $y'$. The sequence is derived from $Y^{(n)}_{k:l}$, which comes from the rule $S \to Y^{(n)}_{k:l}\,  C^{(n)}_s$ and the application of this rule costs $1$. The rest of the blocks of $c'$ are derived from the last three rules and they cost $0$. 
\end{proof}

\section{Manual Analysis of WikiBio Errors}
In Table~\ref{tab:errors} we analyze the faithfulness errors of the \gimodel{} policies on 50 random test examples from the WikiBio dataset, comparing them to the generations of the S2S+Copy model. We divide the errors into hallucination errors, where the model invents facts neither supported nor contradicted by the table, explicit contradiction errors, where the model explicitly contradicts information in the table, and implicit contradiction errors, where the model contradicts information that is only implicit in the table. 

We find that the \gimodel{} model hallucinates at approximately the same rate as does S2S+Copy. However, it also generates more explicit contradictions. These tend to occur when a span containing contradictory information is copied to the canvas, but is not subsequently edited. We suspect that incorporating additional losses, such as a round-trip reconstruction loss~\citep{tu2017neural} will be helpful here. It is notable that the \gimodel{} generations struggle with implicit contradiction more than S2S+Copy. Some examples of implicit contradiction we observed include when a person's gender or nationality are strongly suggested by their name, or their area of study by their thesis title, despite this information not being explicit in the table. We suspect that bigger models, especially if they store state in addition to the canvas (which our \gimodel{} models do not), will better address these cases.

\label{sec:wikierrors}
\begin{table}[t!]
 \small
    \centering
    \begin{tabular}{lcc}
    \toprule
    & \gimodel{} & S2S+Copy \\
    \midrule
    Hallucination & 6 & 5 \\
    Explicit contradiction & 6 & 1 \\
    Implicit contradiction & 3 & 0 \\
    \bottomrule
   \end{tabular}
    \caption{Manual categorization of faithfulness errors made by \gimodel{} and S2S+Copy models on 50 random examples from the WikiBio test set.}
    \label{tab:errors}
\end{table}

\section{Controllability Case Study}
\label{sec:casestudy}
 We briefly consider a case-study that exemplifies controlling generation by controlling the neighbors used at test time. We consider in particular a situation where control is much more easily accomplished under \gimodel{} policies than token-level policies.


Some examples in the E2E dataset consist of only a single sentence (e.g., ``The Golden Curry is a non family friendly Indian restaurant with an average rating located in the riverside area near Cafe Rouge.''), while others split the description into multiple sentences. We consider requiring the generated text to consist of at least 3 sentences, which is interesting and challenging for two reasons. First, only about 8\% of the training examples have $\geq 3$ sentences. Second, while it is sometimes possible to force the generations of token-level models to obey structural constraints by constraining beam search (e.g., by disallowing certain tokens depending on the context), a constrained beam search does not make it easy to guarantee the \textit{presence} of certain structural features. Specifically, while it is easy to constrain beam search so that hypotheses with too many sentences are kept off the beam, it is unclear how to ensure beam search finds only (or even any) hypotheses with enough sentences.

We accordingly restrict both the \gimodel{} and \ltwort{} models to use only neighbors with $ \geq 3$ sentences (as determined by a regular expression) when generating on the E2E development set. We find that 87.2\% of the resulting \gimodel{} generations have $\geq 3$ sentences, while only 73.5\% of the \ltwort{} generations do. Furthermore, the quality of the resulting text remains high, with a ROUGE score of 67.6 for the \gimodel{} generations and 71.7 for \ltwort{}. (Note this comparison unfairly favors \ltwort{}, which generates many fewer of the rare $\geq 3$ sentence generations). 

When the \gimodel{} model fails to respect the constraint it is because it has inserted text that replaces the end of a sentence (or two). We can reach 100\% constraint satisfaction by simply constraining the \gimodel{} model's beam search to never replace a full sentence in the canvas. As noted above, we \textit{cannot} easily constrain the \ltwort{} beam search to reach 100\% constraint satisfaction.

\section{Additional Model and Training Details}
\label{sec:hypers}


Our models are BART~\citep{lewis2020bart}-style encoder-decoder transformers. They consume embeddings of the linearized source tokens $x$ and the current canvas $\hat{y}_{1:M}$ (plus positional embeddings). To allow the model to capture how recently tokens were added to the canvas, we add to each canvas token embedding an embedding of a feature indicating how many time-steps have elapsed since it was added. We also add to each $x$ token embedding the embedding of an indicator feature indicating whether it has been copied to $\hat{y}_{1:M}$. We obtain neighbor embeddings by putting neighbor token embeddings plus positional embeddings plus the embedding of an indicator feature indicating that these are neighbor tokens through the same encoder that consumes $x$.


All transformer encoder-decoders have 6 layers, with model dimension 420, feed-forward dimension 650, 7 attention heads, and dropout rate 0.1. These hyperparameters were chosen (and then fixed) so as to allow the largest model that could be trained in a reasonable amount of time on our GTX 1080 Ti and RTX 2080 Ti GPUs.

We trained with Adam~\citep{kingma2015adam}. We linearly warm-up the learning rate during the first 4,000 training steps, and then use square-root learning rate decay as in \citet{devlin2019bert} after warm-up. To stabilize training we accumulate gradients over 400 target sequences.

We show the training and prediction hyperparameter bounds we considered in Table~\ref{tab:adam}. We selected combinations at random for 50 1-epoch trials for each model, evaluating on validation negative log-likelihood. Our final hyperparameter values are in Table~\ref{tab:finalhypers}.

\begin{table}[t!]
  \small
    \centering
    \begin{tabular}{lc}
    \toprule
        learning rate & \{1e-4, 3e-4, 5e-4, 1e-3\}  \\
        $\beta_1$ &  \{0.85, 0.9, 0.95 \} \\
        $\beta_2$ & \{0.9, 0.99, 0.999 \} \\
        $\eps$ &    \{1e-6, 1e-7, 1e-8 \} \\
        weight decay & \{0, 1e-3, 1e-2\} \\
        beam width & \{1, 5, 10, 20\} \\ 
    \bottomrule
    \end{tabular}
    \caption{Hyperparameter bounds.}
    \label{tab:adam}
\end{table}

\begin{table}[t!]
 \small
    \centering
    \begin{tabular}{@{}lcccccc@{}}
    \toprule
        & LR & $\beta_1, \beta_2$ & $\eps$ & WD & BW \\ 
    \midrule
    E2E-\gimodel{} & 1e-3 & 0.9, 0.999 & 1e-7 & 1e-3 & 5 \\
    E2E-\ltwort{} & 5e-4 & 0.9, 0.999 & 1e-7 & 1e-3 & 5 \\
    E2E-S2S+copy & 5e-4 & 0.9, 0.999 & 1e-7 & 1e-3 & 5 \\
    WB-\gimodel{} & 5e-4 & 0.9, 0.999 & 1e-7 & 1e-3 & 10 \\
    WB-\ltwort{} & 5e-4 & 0.9, 0.999 & 1e-7 & 1e-3 & 10 \\
    WB-S2S+copy & 3e-4 & 0.9, 0.999 & 1e-7 & 1e-3 & 20 \\    
    \bottomrule
    \end{tabular}
    \caption{Final hyperparameters used. ``LR'', ``WD'', and ``BW'' are learning rate, weight decay, and beam width, respectively.}
    \label{tab:finalhypers}
\end{table}



\section{Human Evaluation Details}
\label{sec:heval}
We selected 45 random examples from each of the E2E and WikioBio test-sets for use as crowd-worker prompts. Each example was rated by 3 crowd-workers, and each crowd-worker rated each of the 3 systems. We excluded the 11 responses that did not provide all 9 ratings (3 ratings for each of 3 examples). We show a screen-shot of the questions asked of Mechanical Turk crowd-workers, given a table and generated description, in Figure~\ref{fig:turkers}. We show results of significance tests in Table~\ref{tab:anova}
 and~\ref{tab:tukey}.
 
\begin{table}[t!]
 \small
    \centering
    \begin{tabular}{lc}
    \toprule
    & Pr(>F) \\
    \midrule
    faithfulness & 0.9648 \\
    naturalness & 0.4626 \\
    informativeness & 0.4412 \\
    \midrule
    faithfulness & 0.0182 \\
    naturalness & 0.4689 \\
    informativeness & 0.4360 \\
    \bottomrule
   \end{tabular}
    \caption{System-type $p$-values under ANOVA for E2E (top) and Wikibio (bottom). }
    \label{tab:anova}
\end{table}

\begin{table}[t!]
 \small
    \centering
    \begin{tabular}{@{}lcc@{}}
    \toprule
    & $p$ & CI \\
    \midrule
    \ltwort{}-S2S-faithfulness & 0.9 & (-0.3295, 0.3295)\\
     \ltwort{}-\gimodel{}-faithfulness & 0.9 & (-0.3056, 0.3533) \\
     \gimodel{}-S2S-faithfulness & 0.9 & (-0.3056, 0.3533) \\
    \ltwort{}-S2S-naturalness & 0.6882 & (-0.2325, 0.4706)\\
     \ltwort{}-\gimodel{}-naturalness & 0.6882 & (-0.2325, 0.4706) \\
     \gimodel{}-S2S-naturalness & 0.9 & (-0.3516, 0.3516) \\     
    \ltwort{}-S2S-informativeness & 0.6474 & (-0.4712, 0.2172)\\
     \ltwort{}-\gimodel{}-informativeness & 0.9 & (-0.3918, 0.2965) \\
     \gimodel{}-S2S-informativeness & 0.8337 & (-0.2648, 0.4235) \\          
    \midrule
    \ltwort{}-S2S-faithfulness & 0.9 & (-0.3243, 0.3544)\\
     \ltwort{}-\gimodel{}-faithfulness & 0.2869 & (-0.5574, 0.1213) \\
     \gimodel{}-S2S-faithfulness & 0.2403 & (-0.5724, 0.1062) \\
    \ltwort{}-S2S-naturalness & 0.8601 & (-0.4315, 0.2812)\\
     \ltwort{}-\gimodel{}-naturalness & 0.6329 & (-0.4917, 0.221) \\
     \gimodel{}-S2S-naturalness & 0.9 & (-0.4165, 0.2962) \\
    \ltwort{}-S2S-informativeness & 0.9 & (-0.297, 0.3572)\\
     \ltwort{}-\gimodel{}-informativeness & 0.7419 & (-0.4249, 0.2294) \\
     \gimodel{}-S2S-informativeness & 0.6181 & (-0.4549, 0.1993) \\     
    \bottomrule
   \end{tabular}
    \caption{$p$-value and 95\% confidence intervals under Tukey HSD test (for pairwise difference of means) on E2E (top) and Wikibio (bottom).}
    \label{tab:tukey}
\end{table}

\begin{figure*}[t]
    \centering
    \includegraphics[scale=0.35]{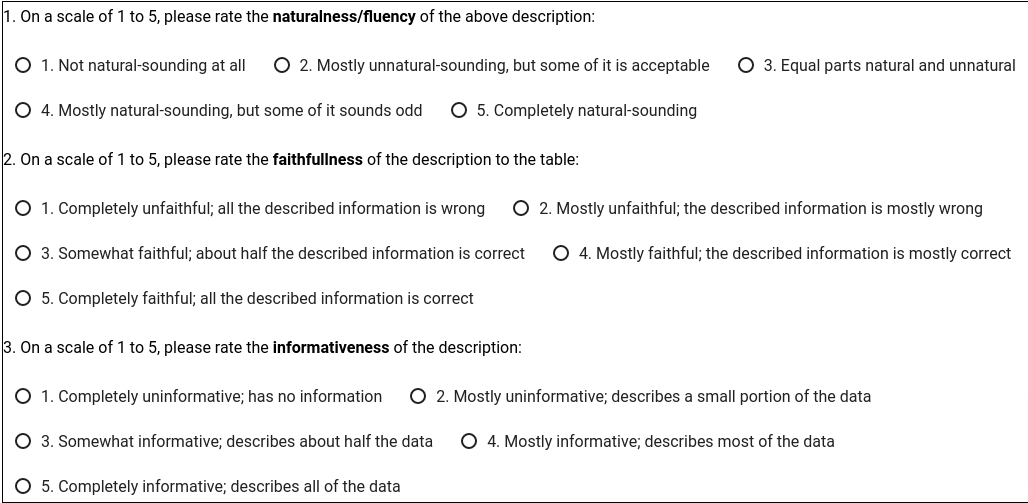}
    \caption{Questions asked of Mechanical Turk crowd-workers, given a table and generated description.}
    \label{fig:turkers}
\end{figure*}

\section{Sample Generations}
\label{sec:samplegens}
We provide 5 random generations from each of \gimodel{}, \ltwort{}, and S2S+copy on the E2E and WikiBio test-sets in Tables~\ref{tab:samples1} and~\ref{tab:samples2}.

\begin{table*}[t!]
 \small
    \centering
    \begin{tabular}{@{}l@{}l@{}}
    \toprule
     & The Blue Spice coffee shop is based near Crowne Plaza Hotel and has a high customer rating of 5 out of 5. \\
 & The Cocum is a pub near Burger King. It has a high customer rating. \\
 & Cocum is a pub near The Sorrento. \\
 & The Giraffe is a restaurant near Rainbow Vegetarian Café in the city centre which serves Fast food. It is not family-friendly. \\
 & The Cricketers is a family friendly coffee shop near Ranch. It has a low customer rating. \\
 \midrule
 & Blue Spice is a coffee shop near Crowne Plaza Hotel. It has a customer rating of 5 out of 5. \\
 & Cocum pub has a high customer rating and is located near Burger King. \\
 & Cocum is a pub near The Sorrento. \\
 & Giraffe is a fast food restaurant near Rainbow Vegetarian Café in the city centre. It is not family-friendly. \\
 & The Cricketers is a family friendly coffee shop near Ranch with a low customer rating. \\
 \midrule
 & Blue Spice is a coffee shop near Crowne Plaza Hotel. It has a customer rating of 5 out of 5. \\
 & Cocum is a highly rated pub near Burger King. \\
 & Cocum is a pub near The Sorrento. \\
 & Giraffe is a fast food restaurant located in the city centre near Rainbow Vegetarian Café. It is not family-friendly. \\
 & The Cricketers is a family friendly coffee shop near Ranch with a low customer rating. \\
\bottomrule
    \end{tabular}
     \caption{Random E2E samples from \gimodel{} (top), \ltwort{} (middle), and S2S+copy (bottom)}.
    \label{tab:samples1}
\end{table*}

\begin{table*}[t!]
 \small
    \centering
    \begin{tabular}{@{}l@{}l@{}}
    \toprule
& 1. morarji desai ( 29 february 1896 -- 10 april 1995 ) was a premiership - winning indian civil servant \\
& \quad who served as prime minister of india between 1977 and 1979 . \\
& 2. charles casali ( 27 april 1923 -- 8 january 2014 ) was a swiss football midfielder who played for switzerland \\
& \quad in the 1954 fifa world cup . \\
& 3. jorgen thalbitzer ( 22 may 1920 -- 29 march 1943 ) was a flying officer of the royal air force during world war ii . \\
& 4. jacob joseph `` jack '' lew ( born august 29 , 1955 ) is an american diplomat . \\
& 5. dietrich - siegwart konrad friedrich fürchtegott von bonin ( 2 february 1917 -- 15 april 1970 ) was a highly decorated \\
& \quad rittmeister der reserves in the wehrmacht during world war ii . \\
\midrule
& 1. morarji desai ( 29 february 1896 -- 10 april 1995 ) was prime minister of india between 1977 and 1979 . \\
& 2. charles casali ( 27 april 1923 -- 8 january 2014 ) was a swiss football midfielder who played for switzerland \\
& \quad in the 1956 fifa world cup . \\
& 3. jorgen thalbitzer ( 22 may 1920 -- 29 march 1943 ) was a danish flying ace during world war ii . \\
& 4. jacob joseph `` jack '' lew ( born august 29 , 1955 ) is the 76th united states secretary of the treasury . \\
& 5. dietrich - siegwart konrad friedrich fürchtegott von bonin ( 2 february 1917 -- 15 april 1970 ) was a highly decorated \\
& \quad rittmeister der reserves in the wehrmacht during world war ii . \\
\midrule
& 1. morarji desai ( 29 february 1896 -- 10 april 1995 ) was the prime minister of india from 24 march 1977 to 15 july 1979 . \\
& 2. charles casali ( 27 april 1923 -- 8 january 2014 ) was a swiss football midfielder who played for switzerland \\
& \quad in the 1950 fifa world cup . \\
& 3. jorgen thalbitzer ( 22 may 1920 -- 29 march 1943 ) was a danish flying ace of world war ii .\\
& 4. jacob joseph `` jack '' lew ( born august 29 , 1955 ) is the 76th united states secretary of state for management and budget . \\
& 5. dietrich - siegwart konrad friedrich fürchtegott von bonin ( 2 february 1917 -- 15 april 1970 ) was a highly decorated \\
& \quad rittmeister der reserves in the wehrmacht during world war ii . \\
 \bottomrule
    \end{tabular}
     \caption{Random WikiBio samples from \gimodel{} (top), \ltwort{} (middle), and S2S+copy (bottom).}
    \label{tab:samples2}
\end{table*}      
\end{document}